\documentclass[11pt,reqno]{article}

\usepackage[T1]{fontenc}
\usepackage[utf8]{inputenc}
\usepackage{lmodern}
\usepackage[margin=1in]{geometry}
\usepackage[protrusion=true,expansion=true]{microtype}

\usepackage{amsmath,amssymb,amsthm,mathtools,bm,bbm}
\usepackage{graphicx, float, subcaption}
\usepackage[hidelinks]{hyperref}
\usepackage[nameinlink,noabbrev]{cleveref}
\usepackage[sort&compress,numbers]{natbib}
\numberwithin{equation}{section}
\theoremstyle{plain}
\newtheorem{theorem}{Theorem}[section]
\newtheorem{proposition}[theorem]{Proposition}
\newtheorem{lemma}[theorem]{Lemma}
\newtheorem{corollary}[theorem]{Corollary}
\theoremstyle{definition}
\newtheorem{definition}[theorem]{Definition}
\theoremstyle{remark}
\newtheorem{remark}[theorem]{Remark}

\crefname{theorem}{Theorem}{Theorems}
\crefname{lemma}{Lemma}{Lemmas}
\crefname{proposition}{Proposition}{Propositions}
\crefname{corollary}{Corollary}{Corollaries}
\crefname{definition}{Definition}{Definitions}
\crefname{remark}{Remark}{Remarks}
\crefname{equation}{Eq.}{Eqs.}
\crefname{figure}{Fig.}{Figs.}
\crefname{section}{Section}{Sections}

\DeclareMathOperator*{\argmax}{arg\,max}

\DeclareMathOperator{\KL}{KL}
\newcommand{\R}{\mathbb{R}}
\newcommand{\1}{\mathbbm{1}}
\newcommand{\ip}[2]{\left\langle #1,\, #2 \right\rangle}

\newcommand{\Hent}{H}
\newcommand{\simplex}{\Delta^{V-1}}

\newcommand{\avgscore}{\bar{s}}

\title{Manifold Trajectories in Next-Token Prediction: \\ From Replicator Dynamics to Softmax Equilibrium}
\author{Christopher R. Lee-Jenkins  \thanks{Centralia College,  600 Centralia College Blvd., Centralia, WA 98531 Email: \texttt{christopher.lee@centralia.edu}}}
\date{\today}

\begin{document}
\maketitle

\begin{abstract}
Decoding in large language models is often described as scoring tokens and normalizing with softmax. We give a minimal, self-contained account of this step as a constrained variational principle on the probability simplex. The discrete, normalization-respecting ascent is the classical multiplicative-weights (entropic mirror) update; its continuous-time limit is the replicator flow. From these ingredients we prove that, for a fixed context and temperature, the next-token distribution follows a smooth trajectory inside the simplex and converges to the softmax equilibrium. This formalizes the common “manifold traversal” intuition at the output-distribution level. The analysis yields precise, practice-facing consequences: temperature acts as an exact rescaling of time along the same trajectory, while top-k and nucleus sampling restrict the flow to a face with identical guarantees. We also outline a controlled account of path-dependent score adjustments and their connection to loop-like, hallucination-style behavior. We make no claims about training dynamics or internal representations; those are deferred to future work.
\end{abstract}

\section{Introduction}

Large language models (LLMs) are typically described as choosing the next token by
scoring a vocabulary and normalizing the scores with a softmax at some temperature.
This account is operationally correct, but it leaves open a conceptual question often
raised in the community: \emph{what does the decoding step look like as a dynamical process?}
In particular, when practitioners say that models ``traverse a manifold'' during decoding,
is this merely a metaphor, or can one make it a theorem, precisely stated and proved from
basic, verifiable ingredients?

This paper gives a minimal, self-contained answer at the level of the \emph{output distribution}
for a fixed context. We show that the next-token distribution is characterized by a classical
free-energy principle and that the normalization-respecting ascent to that equilibrium admits
both (i) a one-line discrete update—entropic mirror ascent, i.e., multiplicative weights—and
(ii) a continuous-time limit—the replicator differential equation on the probability simplex.
From these pieces we obtain a rigorous statement of the manifold picture: the decoding
distribution follows a smooth trajectory in the open simplex and converges to the softmax
equilibrium. The development uses only standard tools (entropy, convex duality, mirror
methods, replicator dynamics); no geometric background is
assumed in the body of the paper.

\paragraph{Contributions.}
\begin{itemize}
\item We provide a derivation of next-token decoding as a \emph{constrained
variational problem} on the probability simplex. The discrete update is a KL-prox mirror step
(multiplicative weights), and the continuous limit is the \emph{replicator} ODE.
\item We prove a \emph{manifold-traversal theorem} (Thm.~\ref{thm:manifold}): for fixed scores and
temperature, the decoding distribution traces a $C^1$ curve inside the open simplex and
converges to softmax. This formalizes a common heuristic—``decoding moves along a manifold’’—
as a theorem in the next-token setting.
\item We extract precise practitioner consequences: (a) \emph{temperature} acts as an exact
\emph{time-rescaling} of the same trajectory; (b) \emph{top-$k$/nucleus} truncation simply
restricts the flow to a face of the simplex with identical guarantees; (c) mild path dependence
in score adjustments can yield loops and brittle attractors, offering a controlled language for
``hallucination''-like behavior.
\end{itemize}

\paragraph{Scope.}We focus exclusively on the \emph{output distribution over next tokens} for a fixed context.
We do not make claims about training dynamics, internal representations, or production decoders
literally integrating an ODE; the ODE appears as a continuous-time limit of the standard
multiplicative update. A broader geometric framework that treats hidden states as coordinates
on a high-dimensional manifold and views the output distribution as a projection is prepared by
the author and deferred to future work (see the pointer in \S\ref{sec:discussion}).

\subsection{Preliminaries (notation and classical facts)}\label{subsec:preliminaries}
Let $V\ge 2$ be the vocabulary size. We write
\[
\simplex=\Big\{p\in\R^V_{\ge 0}:\ \sum_{i=1}^V p_i=1\Big\},\qquad
H(p)=-\sum_{i=1}^V p_i\log p_i\ \ \ (0\log 0:=0).
\]
For a fixed context, let $s\in\R^V$ denote the score/logit vector and $T>0$ the temperature.
Define the temperature-scaled log-partition (log-sum-exp)
\[
A(s)=T\log\!\Big(\sum_{i=1}^V e^{s_i/T}\Big),\qquad
\pi_i(s)=\frac{e^{s_i/T}}{\sum_{j=1}^V e^{s_j/T}}.
\]
Classically, $\pi(s)=\nabla A(s)$ and
$\nabla^2 A(s)=\tfrac{1}{T}\big(\mathrm{diag}(\pi)-\pi\pi^\top\big)$; see
\citet{jaynes1957information,cover2006elements,wainwright2008graphical,murphy2012ml}.
We use $\ip{p}{s}=\sum_i p_i s_i$ and $\1=(1,\dots,1)^\top$.

Two standard viewpoints will be used throughout. First, a \emph{variational} view
(Gibbs’ principle): $\pi$ uniquely maximizes the strictly concave free energy
$\ip{p}{s}+T\,H(p)$ over $\simplex$ \citep{jaynes1957information,cover2006elements,boyd2004convex}.
Second, a \emph{constraint-respecting update} view: the entropic mirror (KL-prox) step
\[
p^{(t+1)}=\argmax_{p\in\simplex}\Big\{\ip{p}{s}+T\,H(p)-\tfrac{1}{\eta}D_{\mathrm{KL}}(p\|p^{(t)})\Big\}
\]
produces the multiplicative-weights update
$p^{(t+1)}_i \propto p^{(t)}_i \exp((\eta/T)s_i)$
\citep{beck2003mirror,arora2012multiplicative,raginsky2012mirror}. As $\eta\to 0$, the
continuous-time limit is the \emph{replicator} ODE
$\dot p_i=\tfrac{1}{T}p_i\big(s_i-\sum_j p_j s_j\big)$ on $\simplex$
\citep{shahshahani1979nouvelle,hofbauer1998evolutionary}. These facts are classical; proofs are
reproduced later for completeness and to keep the paper self-contained.

Adding a constant $c\1$ to $s$ leaves $\pi$ unchanged; scaling $s\mapsto \alpha s$ is equivalent to
$T\mapsto T/\alpha$. Temperature therefore sets the selection pressure and, in our dynamical view,
the time scale of the trajectory we analyze in \S\ref{sec:mirror}–\ref{sec:manifold}.

\paragraph{Positioning with prior work.}
Our development relies only on well-established building blocks: Gibbs’ variational principle for
softmax \citep{jaynes1957information,cover2006elements}, entropic mirror ascent and multiplicative
weights \citep{beck2003mirror,arora2012multiplicative,raginsky2012mirror}, and the replicator flow
with its Shahshahani geometry on the simplex \citep{shahshahani1979nouvelle,hofbauer1998evolutionary}.
What is new here is the \emph{decoding-level formalization}: we prove that the output distribution in
next-token prediction follows a manifold trajectory and give exact, practice-facing corollaries for
temperature, truncation, and mild path dependence. A broader contact geometry perspective is sketched
in \S\ref{sec:discussion} with a pointer to the author’s preprint.

\paragraph{Organization.}
Section~\ref{sec:variational} records the variational characterization and duality facts with proofs.
Section~\ref{sec:mirror} derives the mirror (multiplicative-weights) step, proves ascent, and passes
to the replicator limit. Section~\ref{sec:manifold} states the manifold-traversal theorem and corollaries.
Section~\ref{sec:extensions} discusses temperature, truncation, and controlled path dependence.
Section~\ref{sec:discussion} summarizes limitations and outlines a future extension to hidden-state
manifolds and projection.

\section{A Variational Principle for Next-Token Selection}\label{sec:variational}

Fix a decoding context and vocabulary size $V\ge 2$. Let $s\in\R^V$ denote the (fixed) score/logit vector produced by the model for the next token and $T>0$ the temperature. We work on the probability simplex
\[
\simplex \;=\; \big\{\, p\in\R^V_{\ge 0}\ :\ \sum_{i=1}^V p_i = 1 \,\big\},
\]
and use Shannon entropy $\Hent(p)=-\sum_i p_i\log p_i$ (convention $0\log 0:=0$). Consider the
classical “free-energy” objective from statistical mechanics / information theory,
\begin{equation}\label{eq:free-energy}
\mathcal{F}(p)\;=\;\ip{p}{s}\;+\;T\,\Hent(p)
\;=\;\sum_{i=1}^V p_i s_i \;-\; T\sum_{i=1}^V p_i \log p_i,
\qquad p\in\simplex,
\end{equation}
which is a \emph{concave} functional balancing score maximization against entropy. This setup is standard
(see \citep{jaynes1957information,cover2006elements}; convex-analytic background in \citealp[Ch.~3]{boyd2004convex}).
We record the derivation for completeness and to fix notation.

\subsection{Softmax as the unique maximizer (Gibbs variational principle)}
\begin{proposition}[Softmax variational characterization; standard]\label{prop:softmax-variational}
For fixed $s\in\R^V$ and $T>0$, the concave program
\[
\max_{p\in\simplex}\ \mathcal{F}(p)=\ip{p}{s}+T\,\Hent(p)
\]
has the unique solution
\[
\pi_i \;=\; \frac{e^{s_i/T}}{\sum_{j=1}^V e^{s_j/T}}
\qquad (i=1,\dots,V),
\]
i.e., the softmax distribution at temperature $T$.
\end{proposition}

\begin{proof}
Introduce the Lagrangian for the affine constraint $\sum_i p_i=1$:
\[
\mathcal{L}(p,\lambda)\;=\;\sum_i p_i s_i \;-\; T\sum_i p_i\log p_i \;+\; \lambda\Big(1-\sum_i p_i\Big).
\]
On the relative interior of $\simplex$, stationarity gives, for each $i$,
\[
\frac{\partial \mathcal{L}}{\partial p_i}\;=\; s_i \;-\; T\big(1+\log p_i\big)\;-\;\lambda \;=\; 0
\quad\Rightarrow\quad
\log p_i\;=\;\frac{s_i-\lambda - T}{T}.
\]
Hence $p_i = C\,e^{s_i/T}$ with $C=e^{-(\lambda+T)/T}$. Enforcing $\sum_i p_i=1$ yields
$C=(\sum_j e^{s_j/T})^{-1}$ and thus $p=\pi$. Since $\mathcal{F}$ is strictly concave on the relative interior (Lemma~\ref{lem:strict-concavity}), the maximizer is unique.\ \citep[See also the classical Gibbs principle in][]{jaynes1957information,cover2006elements}.
\end{proof}

\begin{lemma}[Strict concavity]\label{lem:strict-concavity}
The map $\mathcal{F}:\simplex\to\R$ in \eqref{eq:free-energy} is strictly concave on the relative interior of $\simplex$.
\end{lemma}

\begin{proof}
The term $p\mapsto \ip{p}{s}$ is linear. Since $-\,\sum_i p_i\log p_i$ is strictly concave on the relative interior of $\simplex$, $\Hent$ is strictly concave there (and continuous on $\simplex$). Therefore $\mathcal{F}=\ip{p}{s}+T\,\Hent(p)$ is strictly concave on the relative interior, and for $T>0$ the unique maximizer lies in the interior. For background on entropy curvature, see \citet[§3.1.4]{boyd2004convex}.
\end{proof}

\begin{remark}[Normalization and interiority]\label{rem:interior}
For $T>0$, the maximizer $\pi$ satisfies $\pi_i>0$ for all $i$; thus $\pi\in\mathrm{int}\,\simplex$. Intuitively, the entropic term at positive temperature rules out boundary optima unless some $s_i=-\infty$.
\end{remark}

\subsection{Dual (log-sum-exp) viewpoint}
Define the log-partition function (log-sum-exp with temperature)
\begin{equation}\label{eq:logZ}
A(s)\;=\;T\log\!\Big(\sum_{i=1}^V e^{s_i/T}\Big).
\end{equation}
Then $A$ is convex, smooth, and satisfies the Gibbs/Fenchel variational identity
\begin{equation}\label{eq:duality}
A(s)\;=\;\max_{p\in\simplex}\ \Big\{\ \ip{p}{s} \;+\; T\,\Hent(p)\ \Big\},
\end{equation}
a special case of convex duality \citep[cf.][]{rockafellar1970convex} and the classical Gibbs principle \citep{jaynes1957information,cover2006elements}. In exponential-family form,
\[
\nabla A(s)=\pi(s)\quad\text{and}\quad
\nabla^2 A(s)\;=\;\frac{1}{T}\Big(\mathrm{diag}(\pi)\;-\;\pi\pi^\top\Big),
\]
the familiar Fisher information form; see \citet[§3]{wainwright2008graphical}, \citet[§8.3]{murphy2012ml}, \citet[§4.3]{bishop2006pattern}.

\begin{remark}[Smoothness and sensitivity]
The Jacobian $(\partial \pi_i/\partial s_j)_{ij}=\frac{1}{T}\big(\mathrm{diag}(\pi)-\pi\pi^\top\big)$ shows the map $s\mapsto \pi(s)$ is smooth on $\R^V$ and Lipschitz on compact sets, with sensitivity scaling like $1/T$; cf. \citet[§8.3]{murphy2012ml}.
\end{remark}

\subsection{Temperature, limits, and basic properties}
\begin{proposition}[Temperature limits]\label{prop:temperature-limits}
Fix $s\in\R^V$ and let $\pi(T)$ denote the softmax at temperature $T>0$.
\begin{enumerate}
\item As $T\downarrow 0$, $\pi(T)$ concentrates on the set of indices attaining $\max_i s_i$. If the maximizer is unique, then $\pi(T)\to e_{i^\star}$ where $i^\star\in\argmax_i s_i$.
\item As $T\uparrow \infty$, $\pi(T)\to \tfrac{1}{V}\,\1$, the uniform distribution.
\end{enumerate}
\end{proposition}

\begin{proof}
Both statements follow directly from $\pi_i(T)=\exp(s_i/T)/\sum_j\exp(s_j/T)$ and standard Laplace-type asymptotics; see, e.g., \citet[§12]{cover2006elements}.
\end{proof}

\begin{remark}[Score shifts and invariances]\label{rem:invariance}
Adding a constant to all scores $s\mapsto s+c\,\1$ leaves $\pi$ unchanged and shifts $A(s)$ by $c$. Scaling $s\mapsto \alpha s$ is equivalent to rescaling the temperature $T\mapsto T/\alpha$. These invariances let us interpret $T$ as a \emph{selective pressure} parameter.
\end{remark}

We have established a minimal variational principle for decoding: the next-token distribution is the unique maximizer of the strictly concave free energy \eqref{eq:free-energy}, equivalently the gradient of the convex potential $A$ in \eqref{eq:logZ} via \eqref{eq:duality}. In the next section we derive \emph{constraint-respecting updates} that ascend $\mathcal{F}$ on the simplex—the discrete-time mirror (multiplicative-weights) step and its continuous-time limit, the replicator ODE—which inherit the normalization, interiority, and temperature behavior recorded here.

\section{Constraint-Respecting Updates: Mirror $\rightarrow$ Replicator}\label{sec:mirror}

In \S\ref{sec:variational} we showed that for fixed scores $s\in\R^V$ and temperature $T>0$, the next-token distribution is the unique maximizer of the concave free energy
\[
\mathcal{F}(p)=\ip{p}{s}+T\,\Hent(p)\qquad \text{over }\simplex.
\]
We now derive a discrete-time update that \emph{ascends} $\mathcal{F}$ while respecting the simplex constraint, and then pass to a continuous-time limit. Throughout this section, $s$ is fixed (context is frozen); extensions where $s$ depends on $p$ are deferred to \S\ref{sec:extensions}.

\subsection{Entropic mirror ascent equals multiplicative weights}
We use the entropic mirror map (negative Shannon entropy), which yields the classical multiplicative-weights (MW) update; see \citet{beck2003mirror,arora2012multiplicative,raginsky2012mirror}.

\begin{definition}[KL divergence]
For $p,q\in\mathrm{int}\,\simplex$, let $D_{\KL}(p\|q)=\sum_i p_i\log\!\frac{p_i}{q_i}$.
\end{definition}

\begin{proposition}[One-step mirror ascent]\label{prop:mirror-step}
Given $p^{(t)}\in\mathrm{int}\,\simplex$ and step size $\eta>0$, the entropic mirror-ascent step for maximizing $\mathcal{F}$ is
\begin{equation}\label{eq:mirror-prox}
p^{(t+1)} \;=\; \argmax_{p\in\simplex}\ \Big\{\, \ip{p}{s} + T\,\Hent(p)\;-\;\frac{1}{\eta}\,D_{\KL}\!\big(p\|p^{(t)}\big)\,\Big\}.
\end{equation}
Its unique solution has the multiplicative-weights form
\begin{equation}\label{eq:mw}
p^{(t+1)}_i \;=\; \frac{p^{(t)}_i \exp\!\big(\tfrac{\eta}{T}\,s_i\big)}{\sum_{j=1}^V p^{(t)}_j \exp\!\big(\tfrac{\eta}{T}\,s_j\big)}\,,\qquad i=1,\dots,V.
\end{equation}
\end{proposition}

\begin{proof}[Proof (for completeness)]
First-order optimality for \eqref{eq:mirror-prox} with the affine constraint $\sum_i p_i=1$ gives
\[
0 = s_i - T(1+\log p_i) - \tfrac{1}{\eta}\big(\log p_i - \log p^{(t)}_i\big) - \lambda\,,
\]
hence $\log p_i = \log p^{(t)}_i + \tfrac{\eta}{T}s_i - \eta(1+\tfrac{\lambda}{T})$, which yields
$p_i \propto p^{(t)}_i \exp((\eta/T)s_i)$. Normalizing over $i$ gives \eqref{eq:mw}. Strict concavity of the objective in \eqref{eq:mirror-prox} on $\mathrm{int}\,\simplex$ ensures uniqueness.
\end{proof}

\begin{proposition}[Monotone ascent of $\mathcal{F}$]\label{prop:ascent}
The update \eqref{eq:mirror-prox} satisfies
\[
\mathcal{F}\!\big(p^{(t+1)}\big) \;\ge\; \mathcal{F}\!\big(p^{(t)}\big)\;+\;\frac{1}{\eta}\,D_{\KL}\!\big(p^{(t+1)}\|p^{(t)}\big)\;\ge\;\mathcal{F}\!\big(p^{(t)}\big).
\]
In particular, $t\mapsto \mathcal{F}\!\big(p^{(t)}\big)$ is nondecreasing, with equality iff $p^{(t+1)}=p^{(t)}$.
\end{proposition}

\begin{proof}
By optimality of $p^{(t+1)}$ in \eqref{eq:mirror-prox},
\[
\ip{p^{(t+1)}}{s} + T\,\Hent(p^{(t+1)}) - \tfrac{1}{\eta}D_{\KL}\!\big(p^{(t+1)}\|p^{(t)}\big)
\;\ge\; \ip{p^{(t)}}{s} + T\,\Hent(p^{(t)})\,.
\]
Rearrange to obtain the claim.
\end{proof}

\begin{remark}[Interior invariance]
If $p^{(t)}\in\mathrm{int}\,\simplex$ and $s_i\in\R$, then \eqref{eq:mw} yields $p^{(t+1)}\in\mathrm{int}\,\simplex$ and preserves normalization. Thus the update is \emph{constraint-respecting} and stays strictly inside the simplex.
\end{remark}

\subsection{Continuous-time limit: the replicator ODE}
We now pass to a continuous-time limit of \eqref{eq:mw} by letting $\eta\to 0$.

\begin{proposition}[Replicator dynamics]\label{prop:replicator}
Let $p^{(t+1)}$ be given by \eqref{eq:mw}. Then
\[
\frac{p^{(t+1)}_i - p^{(t)}_i}{\eta}\ \Longrightarrow\ \dot p_i \;=\; \frac{1}{T}\,p_i\,\big(s_i-\avgscore\big),\qquad
\avgscore=\sum_{j=1}^V p_j s_j,
\]
as $\eta\to 0$. That is, the continuous-time limit is the replicator ODE on $\simplex$ with ``fitness'' $f_i=s_i/T$. \citep[See, e.g.,][]{hofbauer1998evolutionary,shahshahani1979nouvelle,sandholm2010population}.
\end{proposition}

\begin{proof}[Proof (for completeness)]
Using $\exp(\eta a)=1+\eta a+o(\eta)$,
\[
p^{(t+1)}_i=\frac{p^{(t)}_i\big(1+\tfrac{\eta}{T}s_i+o(\eta)\big)}{\sum_j p^{(t)}_j\big(1+\tfrac{\eta}{T}s_j+o(\eta)\big)}
= p^{(t)}_i\Big(1+\tfrac{\eta}{T}(s_i-\avgscore)+o(\eta)\Big).
\]
Subtract $p^{(t)}_i$, divide by $\eta$, and take $\eta\to 0$.
\end{proof}

The replicator ODE keeps trajectories on the simplex and preserves nonnegativity:

\begin{lemma}[Forward invariance]\label{lem:forward-invariance}
For the ODE $\dot p_i=\tfrac{1}{T}p_i(s_i-\avgscore)$ with $p(0)\in\simplex$, we have $\sum_i p_i(t)=1$ for all $t$, and if $p_i(0)=0$ then $p_i(t)\equiv 0$; in particular, $\mathrm{int}\,\simplex$ and each face of $\simplex$ are forward-invariant.
\end{lemma}

\begin{proof}
Summing gives $\sum_i \dot p_i = \tfrac{1}{T}\sum_i p_i(s_i-\avgscore)=0$. If $p_i(0)=0$, then $\dot p_i(0)=0$ and uniqueness of ODE solutions yields $p_i(t)\equiv 0$.
\end{proof}

\subsection{Lyapunov ascent and convergence (fixed scores)}
Replicator dynamics is a gradient flow of $\mathcal{F}$ with respect to the Shahshahani metric on the simplex \citep{shahshahani1979nouvelle,hofbauer1998evolutionary}, hence $\mathcal{F}$ is a Lyapunov function.

\begin{proposition}[Energy ascent]\label{prop:lyapunov}
Along any solution of $\dot p_i=\tfrac{1}{T}p_i(s_i-\avgscore)$ with fixed $s$,
\[
\frac{d}{dt}\,\mathcal{F}(p(t)) \;\ge\; 0,
\]
with equality iff $p(t)$ is stationary (i.e., $p=\pi$ from Proposition~\ref{prop:softmax-variational}).
\end{proposition}

\begin{proof}[Proof (sketch)]
The replicator vector field equals the (natural) gradient of $\mathcal{F}$ under the Shahshahani inner product
$\langle u,v\rangle_p=\sum_i \tfrac{u_i v_i}{p_i}$ on the tangent space to $\simplex$. Therefore
$\tfrac{d}{dt}\mathcal{F}=\langle \nabla^{\text{nat}}\mathcal{F},\nabla^{\text{nat}}\mathcal{F}\rangle_p\ge 0$,
with equality iff the field vanishes; see \citet{shahshahani1979nouvelle,hofbauer1998evolutionary}.
A direct coordinate proof is also possible by differentiating $\mathcal{F}$ and substituting $\dot p$.
\end{proof}

\begin{corollary}[Convergence to softmax]\label{cor:convergence}
For fixed $s$ and $T>0$, any trajectory with $p(0)\in\mathrm{int}\,\simplex$ satisfies $p(t)\to \pi$ as $t\to\infty$, where $\pi$ is the softmax maximizer of \S\ref{sec:variational}.
\end{corollary}

\begin{remark}[Top-$k$/nucleus as faces of the simplex]
Practical decoders often restrict support before renormalization (top-$k$, nucleus). By Lemma~\ref{lem:forward-invariance}, faces of $\simplex$ are invariant under replicator flow, and the mirror step \eqref{eq:mirror-prox} respects the constraint set. Thus the same analysis applies \emph{verbatim} on any fixed face (subset of tokens), with the softmax equilibrium computed on that face.
\end{remark}

\medskip
In \S\ref{sec:manifold} we use these facts to state and prove the manifold-traversal theorem for decoding trajectories and to connect temperature and truncation to geometric features of the flow.

\section{Manifold Traversal of Decoding Trajectories}\label{sec:manifold}

We now formalize the common intuition that decoding ``moves along a manifold.'' In our setting the manifold is simply the open probability simplex
\[
\mathrm{int}\,\simplex \;=\; \Big\{p\in\R^V_{>0} : \sum_{i=1}^V p_i = 1\Big\},
\]
a smooth $(V-1)$–dimensional submanifold of $\R^V$ with tangent space
\[
T_p\simplex \;=\; \Big\{ u\in\R^V : \sum_{i=1}^V u_i = 0 \Big\}.
\]
The replicator vector field from \S\ref{sec:mirror} is
\begin{equation}\label{eq:repvf}
X_i(p)\;=\;\frac{1}{T}\,p_i\big(s_i-\avgscore\big),\qquad \avgscore=\sum_{j=1}^V p_j s_j,
\end{equation}
which is smooth on $\mathrm{int}\,\simplex$ and satisfies $\sum_i X_i(p)=0$, hence $X(p)\in T_p\simplex$ (tangency).

The following theorem concerns the output distribution for a fixed context; it does not address training dynamics or internal representations.

\subsection{The manifold-traversal theorem}
\begin{theorem}[Manifold traversal of decoding trajectories]\label{thm:manifold}
Fix $s\in\R^V$ and $T>0$. Consider the ODE $\dot p = X(p)$ on $\simplex$ with $X$ given by \eqref{eq:repvf}.
For any initial condition $p(0)\in \mathrm{int}\,\simplex$:
\begin{enumerate}
    \item (\textbf{Well-posed flow on the manifold}) There is a unique global solution $p:[0,\infty)\to \mathrm{int}\,\simplex$ with $p(t)\in\mathrm{int}\,\simplex$ for all $t\ge 0$, and $p(\cdot)$ is a $C^1$ curve on the manifold $\mathrm{int}\,\simplex$.
    \item (\textbf{Ascent of free energy}) The free energy $\mathcal{F}(p)=\ip{p}{s}+T\,\Hent(p)$ from \S\ref{sec:variational} is a Lyapunov function: $\frac{d}{dt}\,\mathcal{F}(p(t))\ge 0$, with equality iff $p(t)$ is stationary.
    \item (\textbf{Equilibrium and convergence}) The unique stationary point in $\mathrm{int}\,\simplex$ is the softmax $\pi$ (Proposition~\ref{prop:softmax-variational}), and $p(t)\to \pi$ as $t\to\infty$.
\end{enumerate}
\end{theorem}

\begin{proof}
(1) By Lemma~\ref{lem:forward-invariance}, $\sum_i p_i(t)=1$ and faces of $\simplex$ are forward-invariant; since $p(0)\in\mathrm{int}\,\simplex$ and $X$ is smooth on $\mathrm{int}\,\simplex$, standard ODE theory yields a unique solution $p(t)\in\mathrm{int}\,\simplex$ for all $t\ge 0$.
(2) Proposition~\ref{prop:lyapunov} shows $\mathcal{F}$ is nondecreasing along trajectories, with strict increase unless stationary.
(3) The stationary condition $\dot p=0$ with $p\in\mathrm{int}\,\simplex$ forces $s_i=\avgscore$ for all $i$ in the support weighted by $p$, which coincides with the KKT conditions for the maximizer of $\mathcal{F}$; by Proposition~\ref{prop:softmax-variational} the unique interior maximizer is $\pi$. Since $\mathcal{F}$ is bounded above by $\mathcal{F}(\pi)$ and strictly increases unless at equilibrium, LaSalle’s invariance principle yields $p(t)\to \pi$. See also the replicator-as-gradient-flow viewpoint in \citet{shahshahani1979nouvelle,hofbauer1998evolutionary}.
\end{proof}

\begin{remark}[Provenance]
The dynamical ingredients—replicator flow on the simplex, its Shahshahani geometry, and free-energy ascent—are classical
(e.g., \citealp{shahshahani1979nouvelle,hofbauer1998evolutionary}). The novelty here is to show that the
\emph{decoding distribution in next-token prediction} obeys this flow, thereby turning the common “manifold traversal”
intuition into a theorem at the output-distribution level, and deriving exact corollaries for temperature (time rescaling)
and truncation (face invariance).
\end{remark}

\begin{remark}
Under the Shahshahani inner product $\langle u,v\rangle_p=\sum_i \frac{u_i v_i}{p_i}$ on $T_p\simplex$, the replicator field is the \emph{natural gradient} of $\mathcal{F}$ \citep{shahshahani1979nouvelle,hofbauer1998evolutionary}; cf.\ information-geometry perspectives \citep{amari1998natural,raginsky2012mirror}. We do not rely on this fact in the main argument, but it clarifies why $\mathcal{F}$ is a Lyapunov function.
\end{remark}

\subsection{Basic geometric corollaries for decoding}
\begin{corollary}[Temperature is a time rescaling]\label{cor:temp-rescale}
For fixed $s$, replacing $T$ by $\alpha T$ in \eqref{eq:repvf} scales the vector field by $1/\alpha$. Thus solutions are reparameterizations of time: lower $T$ yields faster motion along the same trajectory toward the same equilibrium $\pi$ (cf.\ Proposition~\ref{prop:temperature-limits}).
\end{corollary}

\begin{corollary}[Truncation acts on faces]\label{cor:faces}
If decoding restricts support to a fixed subset $S\subseteq\{1,\dots,V\}$ (e.g., top-$k$ or nucleus set) and renormalizes, the dynamics evolve on the face
$\simplex_S=\{p\in\simplex: p_i=0\ \forall i\notin S\}$.
By Lemma~\ref{lem:forward-invariance}, $\simplex_S$ is forward-invariant for \eqref{eq:repvf}, and Theorem~\ref{thm:manifold} holds verbatim with $\pi$ computed on $S$.
\end{corollary}

\begin{remark}[Logit biases and affine invariances]
Adding a constant $c$ to all scores leaves $X$ unchanged; scaling scores $s\mapsto \alpha s$ is equivalent to $T\mapsto T/\alpha$ (Remark~\ref{rem:invariance}). Thus, logit biases that uniformly shift preferences do not change the manifold trajectory, and temperature mediates the speed along that trajectory.
\end{remark}

\medskip
The next section (\S\ref{sec:extensions}) discusses extensions where scores $s$ may weakly depend on $p$ (e.g., simple decoding heuristics) and records conditions ensuring forward invariance and ascent of a generalized free energy.

\section{Phenomenology: Temperature, Truncation, and Path-Dependence}\label{sec:extensions}

We keep the model in \S\ref{sec:variational}--\ref{sec:manifold} and discuss three practical themes:
(i) temperature as a time rescaling of the trajectory; (ii) truncation (top-$k$/nucleus) as an
invariant-face restriction; (iii) path-dependent preferences as a mechanism for ``hallucination''-like
behavior. Throughout, we preserve the ML-facing language of scores and probabilities and avoid any
geometric formalism beyond the simplex.

\subsection{Temperature as selective pressure and effective time}\label{subsec:temp-time}
With fixed scores $s$, the replicator field is $X_i(p)=\tfrac{1}{T}\,p_i(s_i-\bar s)$ (\S\ref{sec:mirror}).
Replacing $T$ by a positive, time-varying schedule $T(t)$ gives
\begin{equation}\label{eq:timevaryingT}
\dot p_i(t)\;=\;\frac{1}{T(t)}\,p_i(t)\,\big(s_i-\bar s(t)\big).
\end{equation}
Let $\tau(t)=\int_0^t \frac{1}{T(u)}\,du$. Then $p(t)=\tilde p(\tau(t))$, where
$\tilde p$ solves the constant-temperature ODE $\frac{d}{d\tau}\tilde p_i=\tilde p_i(s_i-\bar s)$.
Thus temperature schedules reparameterize \emph{time along the same manifold trajectory}
(Corollary~\ref{cor:temp-rescale}). Lower $T$ accelerates motion (higher selective pressure);
higher $T$ slows motion and yields more uniform equilibria (Proposition~\ref{prop:temperature-limits}).

\begin{remark}[Practical reading]
Annealing ($T\downarrow$) sharpens the distribution faster along the same path; cooling too abruptly
can cause early lock-in on spurious modes, while warming ($T\uparrow$) slows convergence and increases diversity.
\end{remark}

\subsection{Truncation as dynamics on a face of the simplex}\label{subsec:truncation}
Let $S\subseteq\{1,\dots,V\}$ be a fixed support set (top-$k$ or nucleus). Define the face
$\simplex_S=\{p\in\simplex: p_i=0\ \forall i\notin S\}$. By Lemma~\ref{lem:forward-invariance},
$\simplex_S$ is forward-invariant for the replicator ODE, and the mirror step
\eqref{eq:mirror-prox} respects $\simplex_S$ by construction. All results from
\S\ref{sec:mirror}--\ref{sec:manifold} hold \emph{verbatim} on $\simplex_S$, with softmax and free energy
computed using indices in $S$ only. Truncation therefore acts as an \emph{ecological bottleneck}:
the same ascent principle operates on a lower-dimensional manifold.

\subsection{Path-dependent preferences and ``hallucination''-like behavior}\label{subsec:path}
Sections \ref{sec:mirror}--\ref{sec:manifold} assumed $s\in\R^V$ fixed (context frozen). In practice, the
scores used during decoding can acquire \emph{path dependence}, e.g., via token-dependent biases,
heuristics modifying logits after each step, or implicit feedback from the partial output. To discuss
this briefly, consider a mild dependence $s=s(p)$ that is Lipschitz on $\mathrm{int}\,\simplex$.
Then the vector field
\[
\dot p_i \;=\; \frac{1}{T}\,p_i\big(s_i(p) - \bar s(p)\big)
\]
is well-defined and tangent to the simplex. When $s$ derives from a potential, the flow remains
a gradient ascent of a generalized free energy; but in general $s(p)$ can include non-potential (curl)
components, producing \emph{nonconservative} dynamics on $\simplex$.

Nonconservative preference fields can yield \emph{loops} and \emph{lock-ins}: along a closed path in $p$,
the net change in alignment need not cancel, permitting cycles or self-reinforcing excursions away from
the original context. Phenomenologically, such path-dependent drifts can present as ``hallucination''-like
behavior: local selection is strong (the replicator field is large in norm), yet global coherence is weak
because successive local preferences are not globally integrable. We deliberately avoid stronger claims:
formal analysis of path-dependent $s(p)$ lies beyond our scope here, see Figure~\ref{fig:projection-intuition} for a schematic of this perspective. 

\begin{figure}[ht]
  \centering
  \begin{subfigure}{0.49\textwidth}
    \centering
    \includegraphics[width=\linewidth]{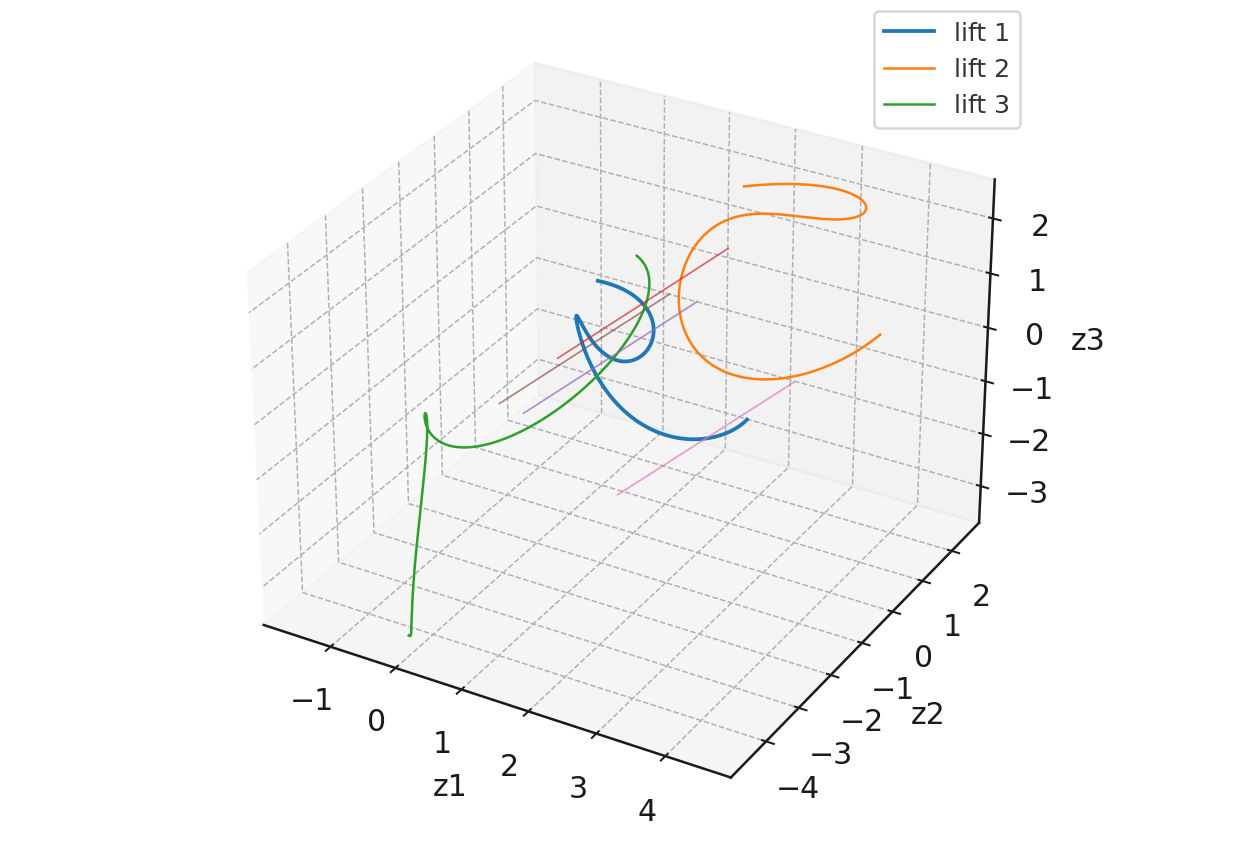}
  \end{subfigure}\hfill
  \begin{subfigure}{0.49\textwidth}
    \centering
    \includegraphics[width=\linewidth]{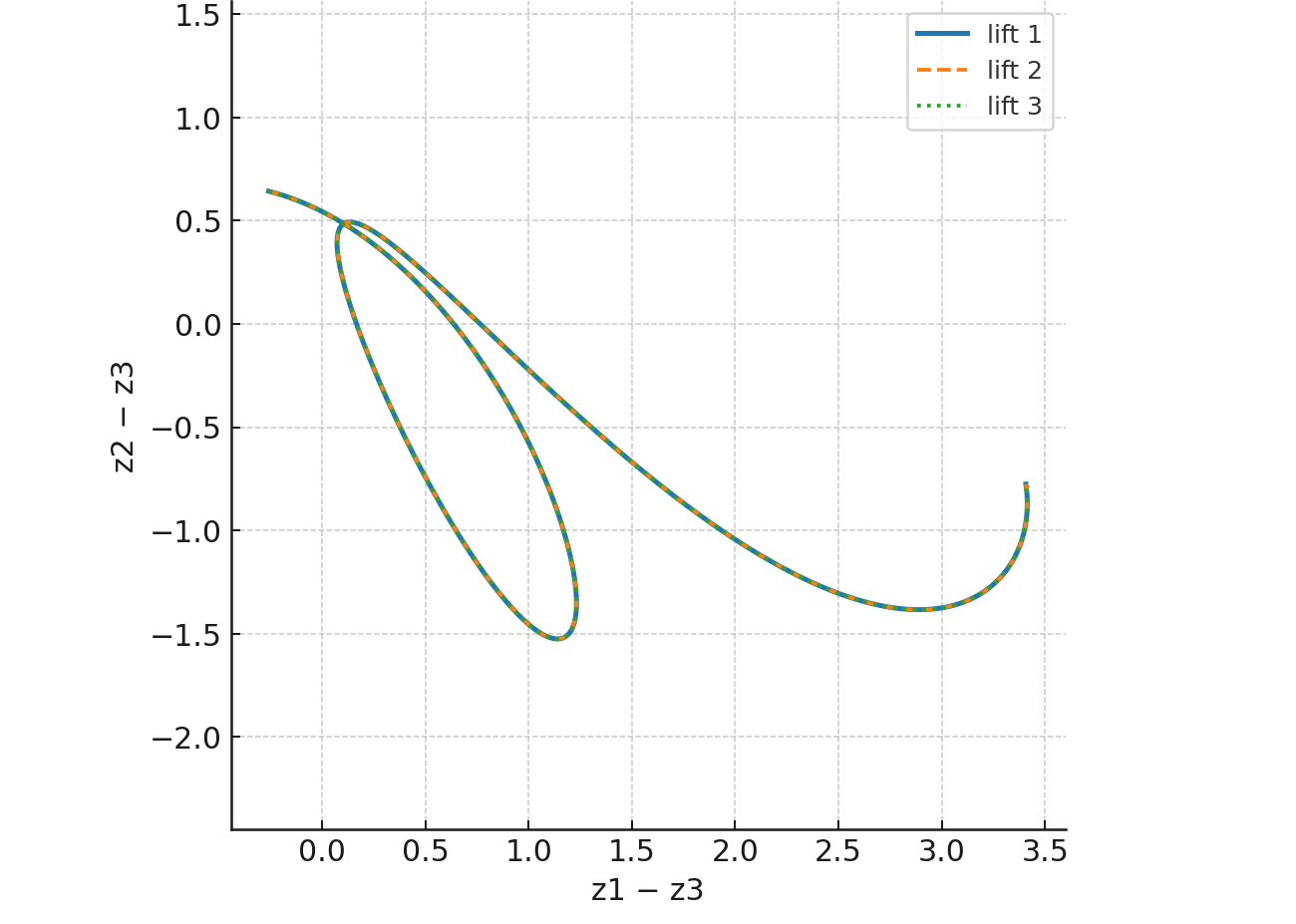}
  \end{subfigure}
  \caption{Left: three internal “lifts”; the faint straight guide lines indicate the fixed \emph{vertical projection direction} toward the simplex plane. Right: projecting each lifted curve \emph{along that vertical direction} yields the same observable trajectory (orange) on the probability simplex.}
  \label{fig:projection-intuition}
\end{figure}

Our results justify the following practical guidance without additional assumptions.
(i) Temperature directly modulates the \emph{speed} along a fixed manifold path
(\S\ref{subsec:temp-time}); (ii) truncation restricts dynamics to a face with identical guarantees
(\S\ref{subsec:truncation}); (iii) when auxiliary heuristics introduce strong path dependence,
the system can develop cycles or brittle attractors. Mitigations include moderating selective pressure
(raising $T$), delaying truncation, or dampening path-dependent logit shifts.

\section{Discussion and Outlook}\label{sec:discussion}

\paragraph{What is new.}
The mathematical ingredients we use—Gibbs' variational principle for softmax, entropic mirror ascent
as multiplicative weights, the replicator flow and its Shahshahani geometry—are classical
(\S\ref{sec:variational}–\ref{sec:manifold}).
Our contribution is to show that \emph{next-token decoding} in LLMs admits a rigorous dynamical
description on the probability simplex, thereby turning the common ``manifold traversal'' intuition
into a theorem at the output-distribution level (Thm.~\ref{thm:manifold}), and to extract exact,
practice-facing corollaries: temperature as a time rescaling (Cor.~\ref{cor:temp-rescale}) and
truncation as restriction to a face with identical guarantees (Cor.~\ref{cor:faces}).

\paragraph{Limitations and scope.}
Our results are restricted to a fixed decoding context and address only the \emph{output distribution}
over next tokens. We deliberately avoid claims about (i) training dynamics, (ii) internal representation
geometry, or (iii) production decoders literally integrating the ODE; the ODE is a continuous-time limit
of the standard multiplicative update. Section~\ref{sec:extensions} sketches safe phenomenology when
scores acquire mild path dependence; a full analysis of nonconservative preference fields is beyond
our scope here.

\paragraph{Hidden states, constraints, and future work.}
While the present work treats only the output distribution, a parallel variational/dynamical viewpoint
can, in principle, be extended to internal states. One may regard hidden activations for a fixed prompt
as coordinates on a high-dimensional manifold shaped jointly by (a) the \emph{user input} and (b) the
\emph{learned parameters}, which impose complementary constraints. The observable token distribution
then arises as a \emph{projection} from this internal state manifold to the probability simplex.

There exists a geometric formalism in which entropy-regularized, normalization-preserving flows arise as \emph{contact}-type dynamics; in that language, the entropic mirror step \eqref{eq:mirror-prox} discretizes a contact descent and the replicator field \eqref{eq:repvf} is its continuous-time limit (see \citealp{bravetti2017contact,geiges2008introduction} for background). A folded-symplectic \emph{bulk} inducing a contact structure on a \emph{fold} and projecting to screen-level, normalized dynamics provides a compact global mechanism that recovers softmax/replicator; see the author's preprint \citep{leejenkins2025folded} for a complete treatment of this bulk–fold–screen construction. In that framework, the contact Reeb direction implements an \emph{effective time} reparameterization, consistent with our temperature-as-time result in \S\ref{subsec:temp-time}. We keep these details out of the main text; a full development is deferred to future work. The goal of such an extension would be to relate internal state flows to the output-level
replicator dynamics established here, providing a principled bridge between representation dynamics and
decoding behavior.

\bibliography{biblio}

\end{document}